\documentclass{article}

\usepackage{microtype}
\usepackage{graphicx}
\usepackage{subfigure}
\usepackage{booktabs} 
\usepackage{xcolor}

\usepackage{hyperref}

\usepackage[accepted]{icml2025}

\usepackage{amsmath}
\usepackage{amssymb}
\usepackage{mathtools}
\usepackage{amsthm}

\usepackage[capitalize,noabbrev]{cleveref}

\theoremstyle{plain}
\newtheorem{theorem}{Theorem}[section]
\newtheorem*{theorem*}{Theorem}

\newtheorem{corollary}[theorem]{Corollary}
\theoremstyle{definition}

\theoremstyle{remark}
\newtheorem{remark}[theorem]{Remark}

\usepackage[textsize=tiny]{todonotes}

\icmltitlerunning{A Unified Framework for Simultaneous Parameter and Function Discovery in Differential Equations}

\begin{document}

\twocolumn[
\icmltitle{A Unified Framework for Simultaneous Parameter and Function Discovery in Differential Equations}



\icmlsetsymbol{equal}{*}

\begin{icmlauthorlist}
\icmlauthor{Shalev Manor}{yyy}
\icmlauthor{Mohammad Kohandel}{yyy}
\end{icmlauthorlist}

\icmlaffiliation{yyy}{Department of Applied Mathematics, University of Waterloo, Waterloo, Ontario, Canada}

\icmlcorrespondingauthor{Shalev Manor}{smanor@uwaterloo.ca}

\icmlkeywords{Scientific Machine Learning, PINN}

\vskip 0.3in
]



\printAffiliationsAndNotice{}

\begin{abstract}
Inverse problems involving differential equations often require identifying unknown parameters or functions from data. Existing approaches, such as Physics-Informed Neural Networks (PINNs), Universal Differential Equations (UDEs) and Universal Physics-Informed Neural Networks (UPINNs), are effective at isolating either parameters or functions but can face challenges when applied simultaneously due to solution non-uniqueness. In this work, we introduce a framework that addresses these limitations by establishing conditions under which unique solutions can be guaranteed. To illustrate, we apply it to examples from biological systems and ecological dynamics, demonstrating accurate and interpretable results. Our approach significantly enhances the potential of machine learning techniques in modeling complex systems in science and engineering.
\end{abstract}

\section{Introduction}

Physics-Informed Neural Networks (PINNs) have emerged as a powerful tool for solving problems governed by partial differential equations (PDEs), offering the ability to incorporate prior knowledge of physical laws directly into the learning process \cite{raissi2019physics}. By embedding differential equation residuals as penalty terms in the loss function, PINNs reduce reliance on large datasets, improve convergence rates, and incentivize model predictions to respect underlying physical principles. These properties have enabled their successful application in domains such as fluid dynamics, structural mechanics, and biomedical modeling \cite{toscano2024pikan, raissi2024extensions}.

Inverse problems, where unknown coefficients in a differential equation are learned from data, represent a common application of PINNs \cite{raissi2020hidden}. In some cases, the governing equations may contain unknown functional terms that represent dynamics not captured in the model. Universal Physics-Informed Neural Networks (UPINNs) \cite{podina2023universal} extend the PINN framework by approximating these terms using neural networks. Universal Differential Equations (UDEs) \cite{rackauckas2020universal}, on the other hand, embed neural networks directly into differential equations to model unknown or poorly understood dynamics. While these methods can address unknown terms, simultaneously learning both unknown constants and functional forms often results in non-unique solutions, which limits their applicability in certain real-world scenarios.

In this work, we address the challenge of solution non-uniqueness in inverse problems involving PINNs. We propose a set of sufficient conditions under which uniqueness can be guaranteed, even when both unknown constants and functions are present in the governing equations. We validate these conditions through three illustrative examples: a toy model of chemotherapy intervention extending prior work by \cite{podina2024chemotherapy}, a modified version of the chemotherapy model with reduced prior information, and the Lotka-Volterra system, a classical predator-prey model in ecology.

Our contributions are threefold. First, we derive sufficient conditions for solution uniqueness when both unknown constants and functions are present in differential equations. Second, we provide rigorous proofs for these conditions, generalizing the applicability of PINNs and UPINNs in solving inverse problems. Third, we demonstrate the effectiveness of our approach on benchmark examples, highlighting its potential for broader applications in scientific machine learning.

The remainder of this paper is organized as follows. In Section \ref{Background}, we outline the challenges associated with combining unknown constants and functional terms. In Section \ref{Theory}, we introduce our proposed sufficient conditions and provide proofs of their validity. Section \ref{Application} details the experimental setup used to validate our approach, while Section \ref{Results} presents the results for each test case. Finally, in Section \ref{Discussion}, we discuss the implications of our findings, limitations, alternative and complementary methods, and directions for future research.

\section{Background and Related Work}\label{Background}

While PINNs and UPINNs excel at solving inverse problems where either coefficients or functional terms are unknown, combining both in a single framework poses significant challenges. When both constants and functional forms are treated as learnable parameters, the solution space often becomes non-unique \cite{raissi2020hidden}. This issue arises because multiple combinations of coefficients and functions can satisfy the same governing equation and dataset, leading to ambiguity in the learned representations.

For example, consider the system described by:
\[
f(x) = \beta g(x) + u(x),
\]
where \(f(x)\) and \(g(x)\) are known, and \(\beta\) and \(u(x)\) are unknown. Subsection \ref{non-uniquness discussion} shows that there exist infinite combinations of \(\beta\) and \(u(x)\) that satisfy the equation. This fundamental challenge limits the applicability of PINNs and related frameworks in real-world scenarios, especially in cases involving complex physical systems.

\subsection{Chemotherapy Intervention Example}

In \cite{podina2024chemotherapy}, the authors use the following system of equations to describe a chemotherapy intervention in the growth of a cancerous tumor:
\begin{eqnarray}\label{eq1}
  &&  \frac{dN}{dt} = \beta N(1 - N) - Cu(N) \\
  &&  \frac{dC}{dt} = -\gamma N C, \nonumber
\end{eqnarray}
where \(N\) represents the population of cancerous cells, \(\beta\) is an unknown growth constant, \(C\) represents the concentration of the chemotherapeutic drug, \(u(N)\) is the unknown drug action at full concentration, and \(\gamma\) is a known constant. 

The training data consists of sample points \((N(t),C(t))\) along a solution to the system for some set of initial conditions and the function multiplied by \(\beta\) is assumed to be known. 

In their work, the unknown term \(C(t)u(N)\) is replaced with a neural network approximation \(u_{\text{net}}(t, N)\), yielding the revised equation:
\begin{equation} \label{u_net}
    \frac{dN}{dt} = \beta N(1 - N) - u_{\text{net}}(t,N).
\end{equation}
This approach encounters problems with solution uniqueness, as described in the next subsection. To address this, the authors split the data collection into two phases. In the first phase, cells were allowed to grow undisturbed, enabling the estimation of the growth term \(\beta\). In the second phase, chemotherapeutic concentrations were introduced, and the data was used to train the neural network approximation of the drug interaction. 

Our approach eliminates the need for phase-based data collection, allowing simpler experimental procedures while maintaining accuracy regardless of the initial presence or absence of chemotherapeutic concentrations.

\subsection{Problem of Uniqueness}
\label{non-uniquness discussion}

Equation \eqref{eq1} can be reduced to the form:
\[
f(x) = \beta g(x) + u(x),
\]
where \(f(x)\) and \(g(x)\) are taken to be known, and \(\beta\) and \(u(x)\) are unknown. Let \(\beta^*\) and \(u^*(x)\) denote the true values used to generate the data, and \(\bar{\beta}\) denote the predicted \(\beta\). It can be shown that the solutions to this equation are not unique. Specifically, by defining:
\[
u(x) = u^*(x) + (\beta^* - \bar{\beta})g(x),
\]
and substituting back, we obtain:
\begin{eqnarray}
    f(x) &=& \bar{\beta} g(x) + u^*(x) + (\beta^* - \bar{\beta})g(x) \nonumber \\
    &=& \beta^* g(x) + u^*(x).
\end{eqnarray}
Thus, for any \(\bar{\beta} \in \mathbb{R}\), there exists a corresponding \(u(x)\) such that the equation is satisfied. Consequently, without additional constraints, it is impossible to recover the original values \(\beta^*\) and \(u^*(x)\) from the data.

\section{Establishing Uniqueness in Solutions} \label{Theory}

Although the generalized problem does not inherently guarantee unique solutions, specific functional forms and conditions can ensure uniqueness. This section introduces sufficient conditions for the uniqueness of solutions and presents detailed proofs.

\subsection{Uniqueness with Known Temporal Scaling}

The following theorems consider a scenario in which the unknown function \( u(x) \) can be expressed as the product of a known function \( C(x) \), which depends on the full state, and an unknown function \( u(y) \), where \( y \) is a non-injective function of \( x \), and \( g \) is a function of \( y \). This case aligns with the structure of the problems studied in \cite{podina2024chemotherapy}, where \(C(x) = x_2\) is part of the training data, \(y = x \cdot \begin{bmatrix} 1 \\ 0 \\ \end{bmatrix} = x_1 \), \(u(y) = u(x_1)\) is the unknown term to be learned, and \(g(y) = y(1-y)\). The term $d(x)$ is introduced to represent an additional known component. The theorem statements below have drawn their structure in direct inspiration from those in \cite{darabi_combining_2025}, where the authors prove the identifiability of an alternative functional form that addresses the problem stated in Section \ref{non-uniquness discussion}
\newpage
\begin{theorem}[Uniqueness with Known Temporal Scaling]
\label{theorem1}
Suppose we have:
\begin{align}
    \dot x &= \begin{bmatrix}
           f_1(x) \\
           f_2(x) \\
           \vdots \\
           f_{q-1}(x) \\
           \beta g(y) + C(x)u(y) + d(x) \\
           f_{q+1}(x) \\
           \vdots \\
           f_n(x)
         \end{bmatrix}
\end{align}
where \( \beta \in \mathbb{R} \) is an unknown constant, and \( u: \mathbb{R}^k \to \mathbb{R} \) is an unknown function. The quantities \( g(y), C(x), d(x) \), and the mapping \( y = H_1(x) \) are all known.
\\
Suppose there exist two points \( x_1, x_2 \in \mathbb{R}^n \) such that \( y_1 = H_1(x_1) = H_1(x_2) = y_2 \), \( C(x_1) \neq C(x_2) \),
and \( g(y_1) = g(y_2) \neq 0 \). Then \(\beta\) and \(u(y)\) are uniquely identifiable for all \(y \in \mathbb{R}\) such that \(C(y) \neq 0 \).
\end{theorem}

\begin{proof}
Let \( \bar{y} = H_1(x_1) = H_1(x_2) \). Then from the \( q^\text{th} \) component of the system:
\\
\begin{align*}
(\dot x_1)_q &= \beta g(\bar{y}) + C(x_1) u(\bar{y}) + d(x_1), \\
(\dot x_2)_q &= \beta g(\bar{y}) + C(x_2) u(\bar{y}) + d(x_2).
\end{align*}
\\
Subtracting the two equations gives:
\[
(\dot x_1 - \dot x_2)_q = (C(x_1) - C(x_2)) u(\bar{y}) + d(x_1) - d(x_2).
\]
\\
Solving for \( u(\bar{y}) \) yields:
\[
u(\bar{y}) = \frac{(\dot x_1 - \dot x_2)_q + d(x_2) - d(x_1)}{C(x_1) - C(x_2)},
\]
which is well-defined since \( C(x_1) \neq C(x_2) \).
\\
Substituting this value back into the first equation gives:
\[
\beta = \frac{(\dot x_1)_q - C(x_1) u(\bar{y}) - d(x_1)}{g(\bar{y})}.
\]
Since \( g(y) \neq 0 \), this is also well-defined.
\\
Now for general \(y\), we have:
\[
(\dot{x})_q = \beta g(y) + C(x)u(y) +d(x)
\]
with \( u(y) \) being the only unknown.
Thus, we can determine \( u(y) \) as:
\[
u(y) = \frac{(\dot{x})_q - \beta g(y) - d(x)}{C(x)}
\]
Thus \(u(y)\) is uniquely determined for all \(y \in \mathbb{R}\) such that \(C(y) \neq 0\).
\end{proof}

\begin{corollary}
\label{corollary1}
for the data sets:
\[
\dot{\chi} = [\dot{x}_1, \dots, \dot{x}_m], \quad \chi = [x_1, \dots, x_m], 
\]
Where \(\dot{\chi}\) and \(\chi\) follow the governing equation of \ref{theorem1} for some \( \beta \) and \( u:\mathbb{R}^k \rightarrow \mathbb{R}\). Suppose that there exist indices \( i,j \in \{1, \dots, m\} \) such that \( y_i = H_1(x_i) = H_1(x_j) = y_j \), \(C(x_i) \neq C(x_j)\), and \( g(y_i) = g(y_j) \neq 0\).  
Then for the system described in Theorem \ref{theorem1}, \( \beta \) is determined exactly and \( u(y_i) \) is determined for all \(i \in \{1,\dots , m\}\).
\end{corollary}


\subsection{Generalization to Unknown Growth Terms}

The second theorem generalizes the problem by treating the growth term \(g(x)\) as completely unknown. While this introduces additional complexity, uniqueness can still be ensured under stricter conditions.

\begin{theorem}[Uniqueness with Unknown Growth Terms]
\label{theorem2}
Let the governing equation be:
\[
\dot x = \begin{bmatrix}
f_1(x) \\
f_2(x) \\
\vdots \\
f_{q-1}(x) \\
g(y) + C(x)u(y) + d(x) \\
f_{q+1}(x) \\
\vdots \\
f_n(x)
\end{bmatrix},
\]
where \( y = H_1(x) \in \mathbb{R}^k \) and all functions \( C(x), d(x), H_1 \) are known.
\\
If there exist \( x_1, x_2 \in \mathbb{R}^n \) such that:
\[
H_1(x_1) = H_1(x_2) = Y, \quad \text{and} \quad C(x_1) \neq C(x_2),
\]
then the values of \( u(Y) \) and \( g(Y) \) are uniquely determined.
\end{theorem}

\begin{proof}
Let \( y_1 = y_2 = Y \). From the \(q^\text{th}\) component of the system, we have:
\begin{align*}
(\dot{x}_1)_q &= g(Y) + C(x_1) u(Y) + d(x_1), \\
(\dot{x}_2)_q &= g(Y) + C(x_2) u(Y) + d(x_2).
\end{align*}
\\
Subtracting gives:
\[
(\dot{x}_1 - \dot{x}_2)_q = (C(x_1) - C(x_2)) u(Y) + d(x_1) - d(x_2).
\]
\\
Solving for \( u(Y) \):
\[
u(Y) = \frac{(\dot{x}_1 - \dot{x}_2)_q + d(x_2) - d(x_1)}{C(x_1) - C(x_2)},
\]
which is valid since \( C(x_1) \neq C(x_2) \).
\\
Substituting back into either equation (e.g., the first) gives:
\[
g(Y) = (\dot{x}_1)_q - C(x_1) u(Y) - d(x_1),
\]
which completes the proof of identifiability for both \( u(Y) \) and \( g(Y) \).
\end{proof}

\begin{corollary}
\label{corollary2}
With the data sets:
\[
\dot{\chi} = [\dot{x}_1, \dots, \dot{x}_m], \quad \chi = [x_1, \dots, x_m], 
\]
Where \(\dot{\chi}\) and \(\chi\) follow the governing equation of \ref{theorem2} for some \(g, u:\mathbb{R}^k \rightarrow \mathbb{R}\). For the system described in Theorem \ref{theorem2}, \( g(y_i)\) and \( u(y_i) \) can be uniquely determined for all \( i \in \{1, \dots, m\} \) such that there exists a \(j \in \{1, \dots, m\} \) such that \( y_i = H_1(x_i) = H_1(x_j) = y_j \), and \(C(x_i) \neq C(x_j)\).  
\end{corollary}



The two theorems presented above establish conditions under which unique solutions to the inverse problem can be guaranteed for systems involving unknown constants and functions. Theorem \ref{theorem1} applies to cases where the growth term is known up to a multiplicative constant and coupled with Corollary \ref{corollary1}, guarantees uniqueness for all points observed in the training set, while Theorem \ref{theorem2} and Corollary \ref{corollary2} generalize the problem to include completely unknown growth terms with guaranteed uniqueness for specific observations in the dataset.

Although these theorems are promising, one of their assumptions is overly restrictive in practice. Specifically, the condition that two \( y \)-values in a finite dataset are exactly equal is highly improbable due to discretization and finite machine precision. In practical applications, exact equality between two \( y_i \) values is unlikely, which motivates the development of approximate uniqueness results based on continuity assumptions.
To this end, we introduce two more theorems, which we prove in the Appendix.

\begin{theorem}
\label{theorem3}
With the data sets:
\[
\dot{\chi} = [\dot{x}_1, \dots, \dot{x}_m], \quad \chi = [x_1, \dots, x_m], 
\]
For the system described in Theorem \ref{theorem1}, suppose that \( U \) is Lipschitz continuous with Lipschitz constant \( L \).  
Furthermore, suppose that we have \( i,j \in \{1, \dots, m\} \) and \( D \geq 0 \) such that \( \|y_i - y_j\| \leq D \), \( C(x_i) \neq C(x_j) \), and \( g(y_i) \neq 0 \), then \(\beta\) and \(u(y)\) are guaranteed to be within intervals with radii continuously dependent on \(D\).
\end{theorem}

\begin{remark}
\label{remark1}
We find that if the predicted values for \(\beta\) and \(u(y)\) in Theorem \ref{theorem3} (defined as \(\bar{\beta}\) and \(\bar{u}(y)\)) are taken to be in the centers of their respective possible intervals of existence, then the errors in these predictions are taken to be bounded by the following:

\[
|\beta - \bar{\beta}| \leq \bigg|\frac{C(x_i)C(x_j)LD}{g(y_i)(C(x_i)-C(x_j)) - (C(x_i)-C(x_j))}\bigg|
\]
\[
|u(y_p) - \bar{u}(y_p)| \leq \bigg|\frac{g(y_p)}{C(z_p)}\bigg||\beta - \bar{\beta}|
\]
Where \(i\) and \(j\) are the indices in the data set that minimize the above inequality and \(p\) is an index in the set \(\{1,\dots,m\}\). The derivation for these inequalities is included in the proof of Theorem \ref{theorem3} in the Appendix.

\end{remark}

\begin{theorem}
\label{theorem4}
With the data sets:
\[
\dot{\chi} = [\dot{x}_1, \dots, \dot{x}_m], \quad \chi = [x_1, \dots, x_m], 
\]
For the system described in Theorem \ref{theorem2}, suppose that \(g(y)\) and \( u(y) \) are Lipschitz continuous with Lipschitz constants \( L_1\) and \(L_2\), respectively.  
Furthermore, suppose that we have \( i,j \in \{1, \dots, m\} \) and \( D \geq 0 \) such that \( \|y_i - y_j\| \leq D \) and \( C(x_i) \neq C(x_j) \), then \(g(y_i)\) and \(u(y_i)\) are guaranteed to be within intervals with radii continuously dependent on \(D\).
\end{theorem}

\begin{remark}

With a similar setup to Remark \ref{remark1} for the predicted values of \(g(y)\) and \( u(y) \), we find that the errors are bounded by the following:
\[
|u^*(y_i) - \bar{u}(y_i)| \leq \bigg|\frac{D(L_1 + C(x_j)L_2)}{C(x_i) - C(x_j)}\bigg|
\]
\[
|g^*(y_i) - \bar{g}(y_i)| \leq \bigg|\frac{C(x_i)D(L_1 + C(x_j)L_2)}{C(x_i) - C(x_j)}\bigg|
\]
the derivation is shown in the Appendix.
\end{remark}

The additional Theorems demonstrate that the results of Theorems \ref{theorem1} and \ref{theorem2} can be practically applied even when data is not perfectly precise, with bounded error.
These results provide a foundation for designing systems and experiments where unique parameter estimation is essential.


%

\section{Application of Uniqueness Conditions} \label{Application}

To validate the theoretical results presented in Section \ref{Theory}, we applied the uniqueness conditions to three test cases: two variations of a chemotherapy intervention  model and the Lotka-Volterra predator-prey system. This section describes the problem setup, experimental configurations, and results for each case.

In Sections 4.1-4.3, we demonstrate the uniqueness of solutions in a simplified version of the problem. In this version, we provided the true values of the derivatives of a set of training points. This is done to emulate the assumptions made by the theorems exactly. Thus, in these experiments, the only terms that were learned were the unknown functions and constants present in the differential equations.

In Sections 4.4 and 4.5, the experiments are performed on UPINNs. The only data available for the model are samples of \((N(t),C(t))\) obtained from a single trajectory. For these trials, the derivative values were obtained by auto-differentiation of a UPINN fit to the training points. These experiments are intended to demonstrate the practical effectiveness of our approach.

Throughout Sections 4.1-4.3 below, we use a feedforward neural network with a single input, four hidden layers, each with 20 neurons, and a single output. Optimization was performed using the Adam optimizer with a learning rate of \(10^{-3}\). Training was conducted for 1000 epochs in the first and third experiments, and 5000 epochs in the second experiment.

The network outputs were evaluated using the Mean Squared Error (MSE) loss.

All networks discussed below utilize tanh activation functions.

\subsection{Case 1: Chemotherapy Intervention Problem}

The first test case is based on a chemotherapy intervention model described by \cite{podina2024chemotherapy}, with the following equations:
\begin{eqnarray}
&&    \frac{dN}{dt} = \beta N(1 - N) - C(t) u(N) \nonumber \\
&&    \frac{dC}{dt} = -\gamma N C + I(t)
\end{eqnarray}
Here \(N(t)\) is the population of cancerous cells \(C(t)\) is the concentration of the chemotherapeutic drug, \(\beta\) is the unknown growth rate, \(u(N)\) is the unknown drug action, and \(I(t)\) is the external drug injection, modeled as \(I(t) = e^{-\tau(t-4)^2}\) with \(\tau = 5\). In this scenario, \(I(t)\) introduces a burst of chemotherapeutic drug at \(t = 4\), and drug concentration remains nonzero throughout the trial.

The approach used in \cite{podina2024chemotherapy} would struggle to recover the original values under such conditions. By conforming to the uniqueness conditions in Theorem \ref{theorem3}, we train the network to approximate \(u(N)\) instead of the product \(C(t)u(N)\). This ensures that the model learns the unique solution.

The following initial conditions and parameters were used: \(C(0) = 1\), \(N(0) = 1\), and \(\gamma = 0.3\). Two different true forms for \(u(N)\) are considered \(u(N) = N\) and \(u(N) = N^2\), and \(\beta^* = 1\) as the true value, with an initial guess of \(\beta = 2\).

The model successfully recovered \(\beta\) and \(u(N)\) for both functional forms of \(u(N)\). By meeting the conditions in Theorem \ref{theorem1}, the model was able to converge to the correct unique solution.

\subsection{Case 2: Modified Chemotherapy Intervention Problem}
In this modified version, the growth term is treated as completely unknown, leading to the following system:
\begin{equation}
    \frac{dN}{dt} = \Psi(N) - C(t) u(N)
\end{equation}
where \(\Psi(N)\) is the unknown growth term and \(u(N)\) is the unknown drug action.

In this setup, the model must learn both \(\Psi(N)\) and \(u(N)\) simultaneously. Observations indicate that low drug concentrations lead to an increase in the cancer cell population (\(N\)), while high drug concentrations result in a decrease. These dynamics ensure that the uniqueness conditions in Theorem \ref{theorem4} are satisfied. The same network configuration and training procedures were used as in the previous case. To accurately approximate both functions, two separate models were trained for \(\Psi(N)\) and \(u(N)\), respectively.

The model successfully converged to the unique solution for both \(\Psi(N)\) and \(u(N)\), demonstrating that the conditions in Theorem \ref{theorem2} provide sufficient guarantees for uniqueness.

The following initial conditions and parameters were used: \(C(0) = 0.1\), \(N(0) = 0.01\), and \(\gamma = 0.5\). The true values for \(u(N)\) and \(\beta\) were set to \(u(N) = N\) and \(\beta^* = 1\).

\subsection{Case 3: Lotka-Volterra Predator-Prey System}

The Lotka-Volterra system models predator-prey interactions using the following equations:
\begin{eqnarray}
    &&  \frac{dx}{dt} = \alpha x - \beta x y \nonumber \\
    &&  \frac{dy}{dt} = \delta x y - \gamma y
\end{eqnarray}
Here \(x(t)\) is the prey population, \(y(t)\) is the predator population,
\(\alpha, \gamma\) are unknown parameters, and \(\beta x\) and \(\delta y\) are treated as unknown functional terms.

This system exhibits oscillatory behavior, which ensures sufficient variation in the dataset to satisfy the uniqueness conditions in Theorem \ref{theorem3}. In fact, to guarantee the uniqueness of both components of the system, we apply Theorem \ref{theorem3} twice. It is applied independently for each component. The following initial conditions and parameters were used to generate synthetic data: \(x(0) = 2\), \(y(0) = 4\). The true values of the parameters were set to \(\alpha = \beta = \delta = \gamma = 1\). The initial guesses for the parameters were: \(\alpha = 2,\gamma = 2\).

\subsection{Case 4: Lotka-Volterra Predator-Prey System with Proportional Noise}
\label{proportionalNoise}
To test the robustness of UPINNs that utilize this functional form, we tested the performance of a UPINN trained on the same data as in Section 4.3 with injected noise. Each data point was perturbed by up to a specified percentage of its magnitude. perturbations were drawn from a uniform distribution. The UPINN is composed of three neural networks. Two networks to represent the unknown components of the system of differential equations and a single network approximating the solution trajectory. All networks had a single input and four hidden layers of width 20. The networks representing the unknown components of the system had a single output, and the network approximating the trajectory had two outputs. One for each component of the state.

The system was assigned parameters and initial conditions identical to Section 4.3. The initial guesses for the parameters were: \(\alpha = 1.5,\gamma = 0.5\). Throughout training, the loss provided to the optimizer took the form \(\mathcal{L}_{tot} = \mathcal{L}_{Data} + \omega_{DE}\mathcal{L}_{DE}\) where \(\omega_{DE}\) was assigned a value of 0.1. This was done to insure that the general shape of the trajectory would be learned first, with \(\mathcal{L}_{DE}\) helping to make adjustments to satisfy the ODE.

\subsection{Case 5: Chemotherapy Intervention Problem with Varying Dataset Length}
\label{Varying dataset Lengths}
To test the performance of the model in different data regimes, we trained a UPINN on the problem described in Section 4.1 with differing training set lengths. The UPINN was composed of two neural networks. One network composed of a single input taking in \(t\), three hidden layers of width 20, and two outputs approximating \(N(t)\) and \(C(t)\). The \(C(t)\) output was introduced to allow for the ODE loss to be evaluated at times when a data point is not present. The \(C(t)\) term had an MSE data loss and ODE residual loss with no unknown components, making it a PINN. The second network had a single input taking in \(N\), two hidden layers of width 10, and a single output approximating \(u(N)\). The network sizes were reduced from previous trails as it was found that this did not greatly reduce performance and meaningfully reduced training time. 

The model was trained on datasets of length 1024, 512, 256, 128, 64, 32, 16, 8, and 4.

While the number of points where the data loss can be evaluated decreases with the number of data points present in the training set, the ODE loss can still be evaluated for arbitrary values of time with collocation points. Thus, for all trials, the ODE loss used 1024 collocation points when optimizing the ODE loss.

The term \(I(t)\) was modified to increase the introduced amount of chemotherapeutic. It was changed from \(I(t) = e^{-\tau(t-4)^2}\) to \(I(t) = 5e^{-\tau(t-4)^2}\) with \(\tau = 5\) in both cases. The following initial conditions and parameters were used: \(C(0) = 0.1\), \(N(0) = 0.01\), and \(\gamma = 0.3\). The true form for \(u(N)\) was taken as \(u(N) = 2N\), and \(\beta^* = 2\) as the true value, with an initial guess of \(\beta = 1.5\). As in Section \ref{proportionalNoise}, the loss provided to the optimizer took the form \(\mathcal{L}_{tot} = \mathcal{L}_{Data} + \omega_{DE}\mathcal{L}_{DE}\) where \(\omega_{DE}\) was assigned a value of 0.001.

\section{Results} \label{Results}

This section presents the outcomes of applying the proposed uniqueness conditions to five test cases: the chemotherapy intervention problem, its modified version, the Lotka-Volterra predator-prey system, and full UPINN implementations testing robustness to noise and varying dataset lengths. Each subsection provides the problem setup, network performance, and observations.

\subsection{Chemotherapy Intervention Problem}

The solutions for both trials of the chemotherapy intervention problem are shown in Figure \ref{fig:chemotherapy problem results}. 

In the first trial, where the true drug action was \(u(N) = N\), the network achieved a final loss of \(1.65649 \times 10^{-6}\). The predicted value of the growth constant was \(\beta = 1.0003069\). The comparison between the true and predicted drug action is displayed in the top-right panel of Figure \ref{fig:chemotherapy problem results}. The network performed well in approximating the drug action, but accuracy declined for values of \(N < 0.3\). This can be attributed to the absence of training data in this region, as the network only trains on observed data.

In the second trial, where the true drug action was \(u(N) = N^2\), the same network architecture and initial conditions were used. The final loss of the network was \(7.94170 \times 10^{-7}\), with a predicted value of \(\beta = 1.0002422\). The comparison between the true and predicted drug action is shown in the bottom-right panel of Figure \ref{fig:chemotherapy problem results}. For this trial, the divergence between the true and predicted values occurred for \(N < 0.6\), which corresponds to the lower limit of the observed data. This limitation highlights the sensitivity of the network to the availability of training data.

\begin{figure}
\centering
\includegraphics[width=0.49\linewidth]{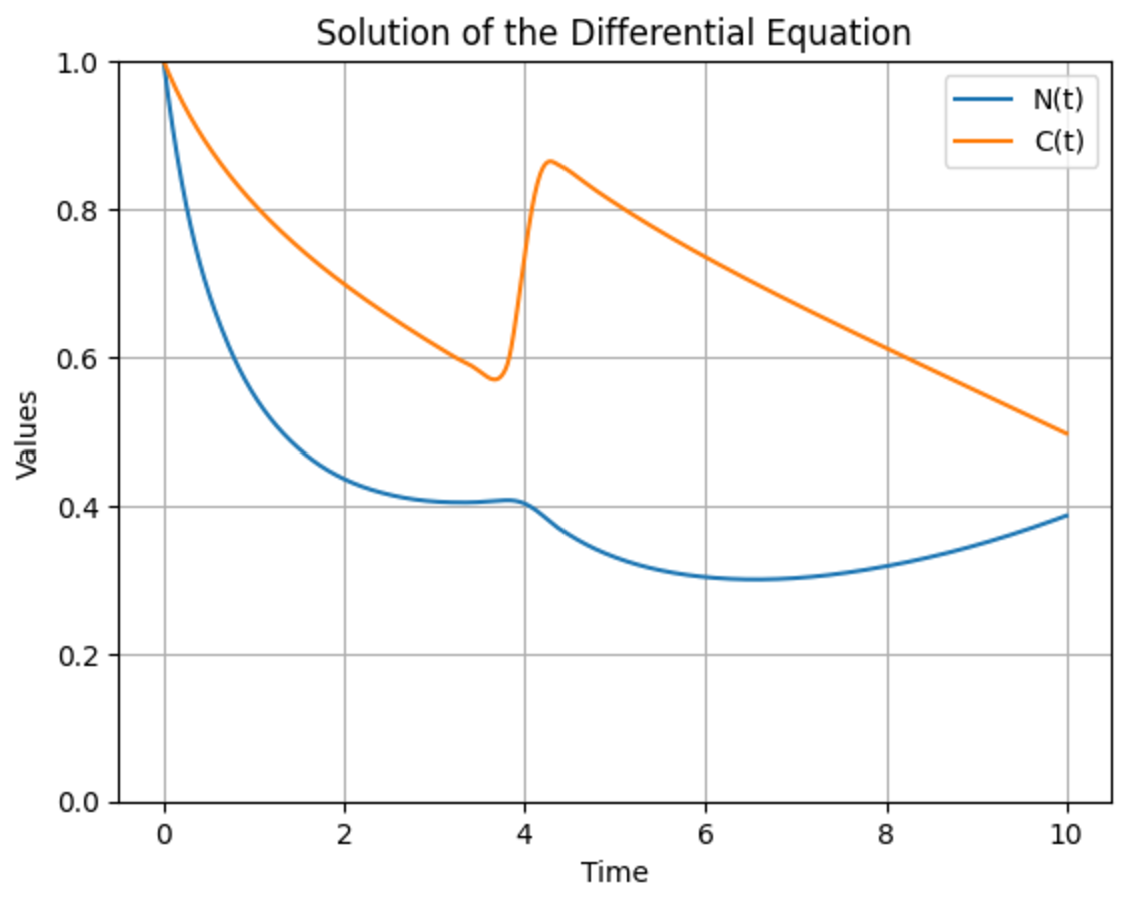}
\includegraphics[width=0.49\linewidth]{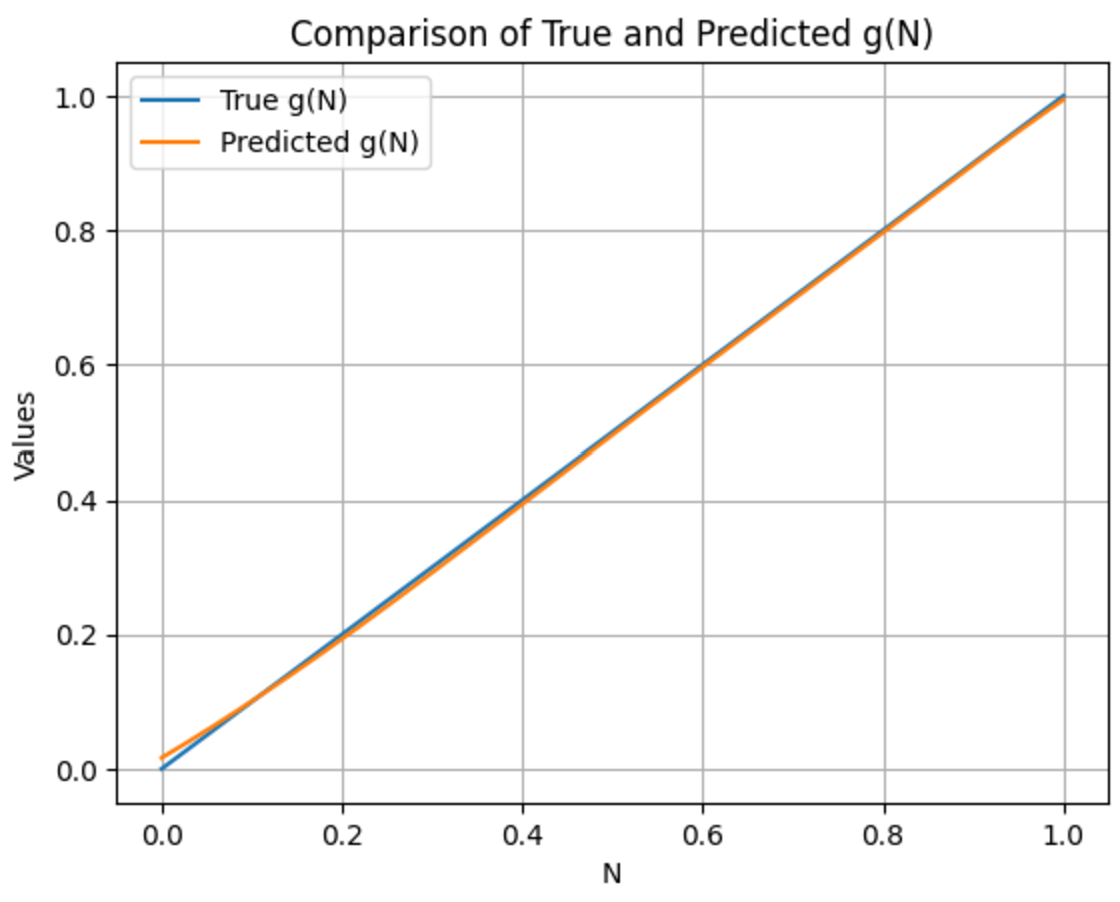}
\includegraphics[width=0.49\linewidth]{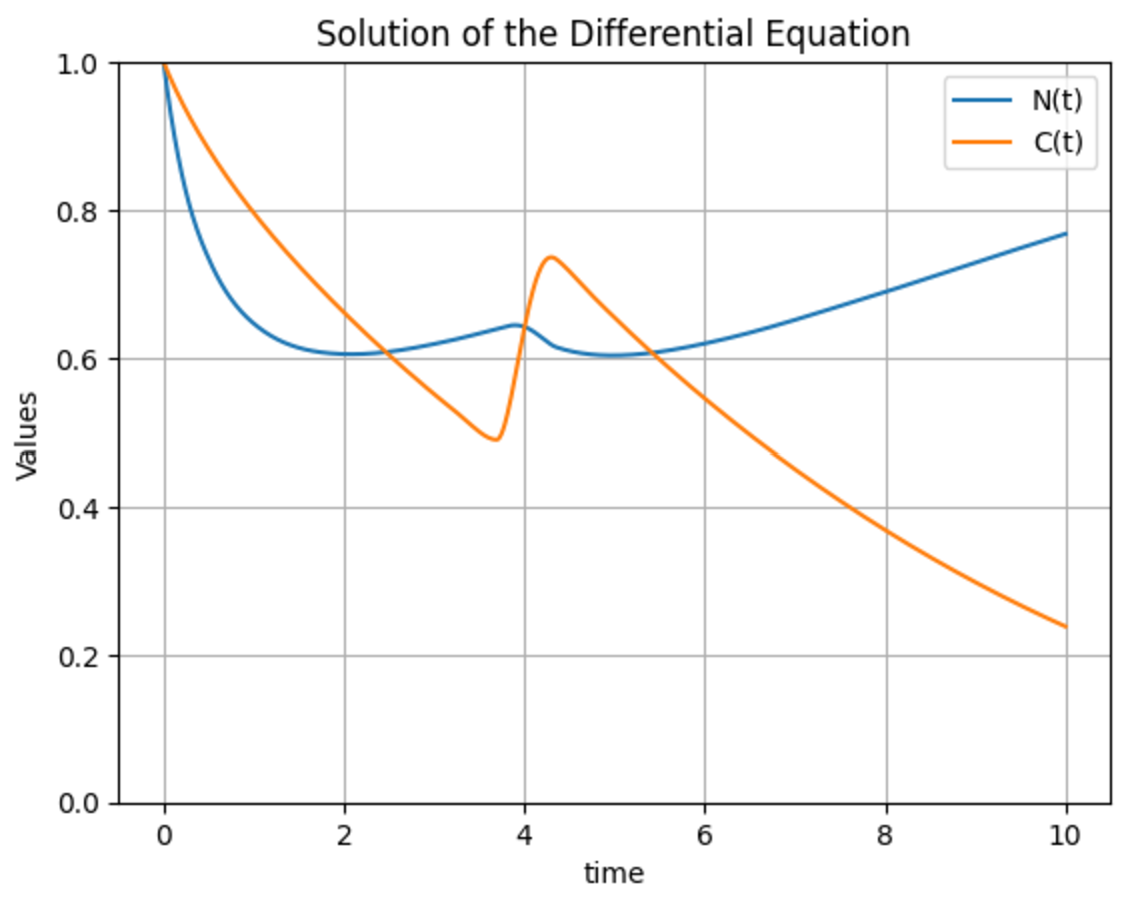}
\includegraphics[width=0.49\linewidth]{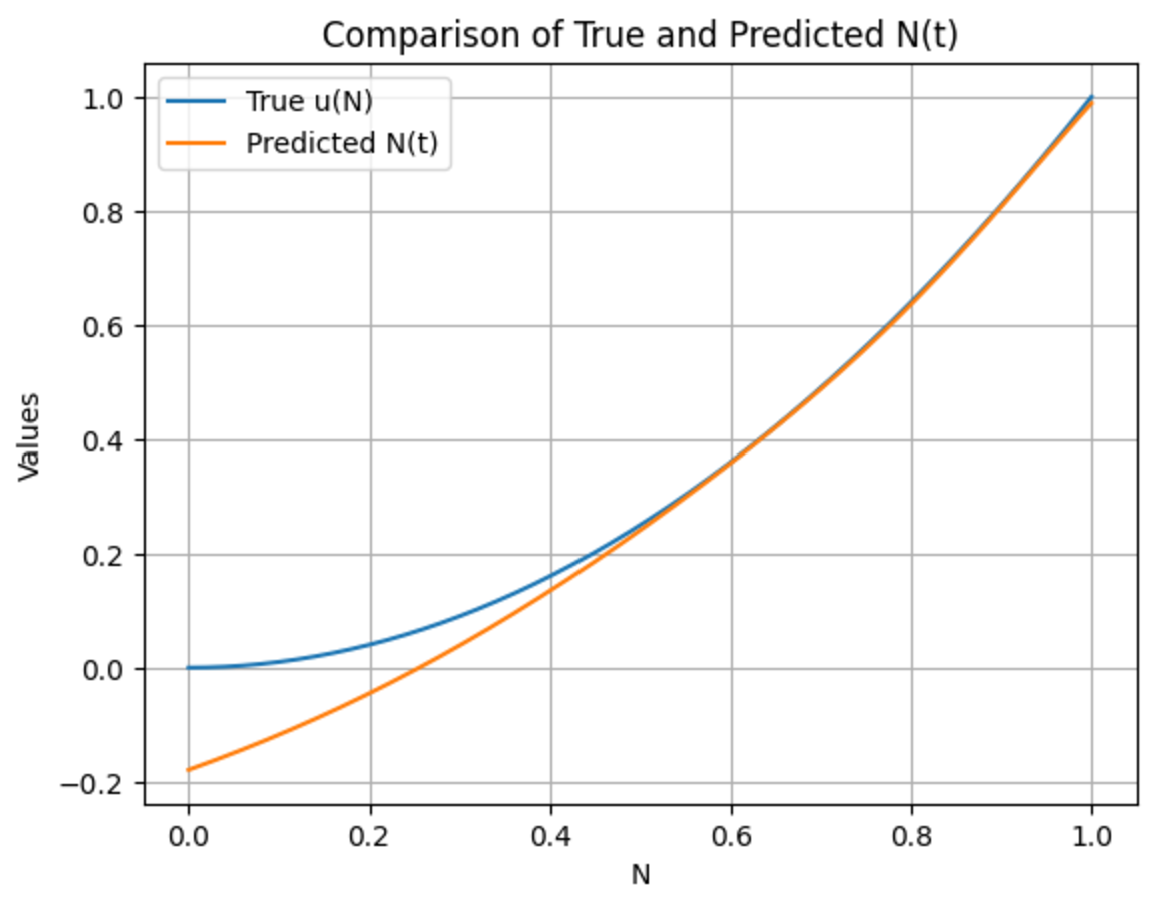}
\caption{\label{fig:chemotherapy problem results} Top row: Solutions to the differential equations and the comparison between true and predicted drug action for Trial 1 (\(u(N) = N\)). Bottom row: Results for Trial 2 (\(u(N) = N^2\)).}
\end{figure}

\subsection{Modified Chemotherapy Intervention Problem}

The modified chemotherapy intervention problem was used to test the ability of the model to simultaneously learn both the unknown growth term \(\Psi(N)\) and the drug action \(u(N)\). The final loss achieved by the network was \(3.211754 \times 10^{-7}\), and the results are shown in Figure \ref{fig:modified problem solution}.

The comparison between the true and predicted values indicates that the network generally performed well. However, a divergence was observed for \(N \approx 0.5\), where the uniqueness conditions are no longer fully satisfied. Interestingly, the predicted growth term \(\Psi(N)\) closely matched the true value, even in regions where the conditions were not met. This unexpected behavior warrants further investigation to understand the factors influencing the network’s performance.

\begin{figure}[h]
\centering
\includegraphics[width=0.49\linewidth]{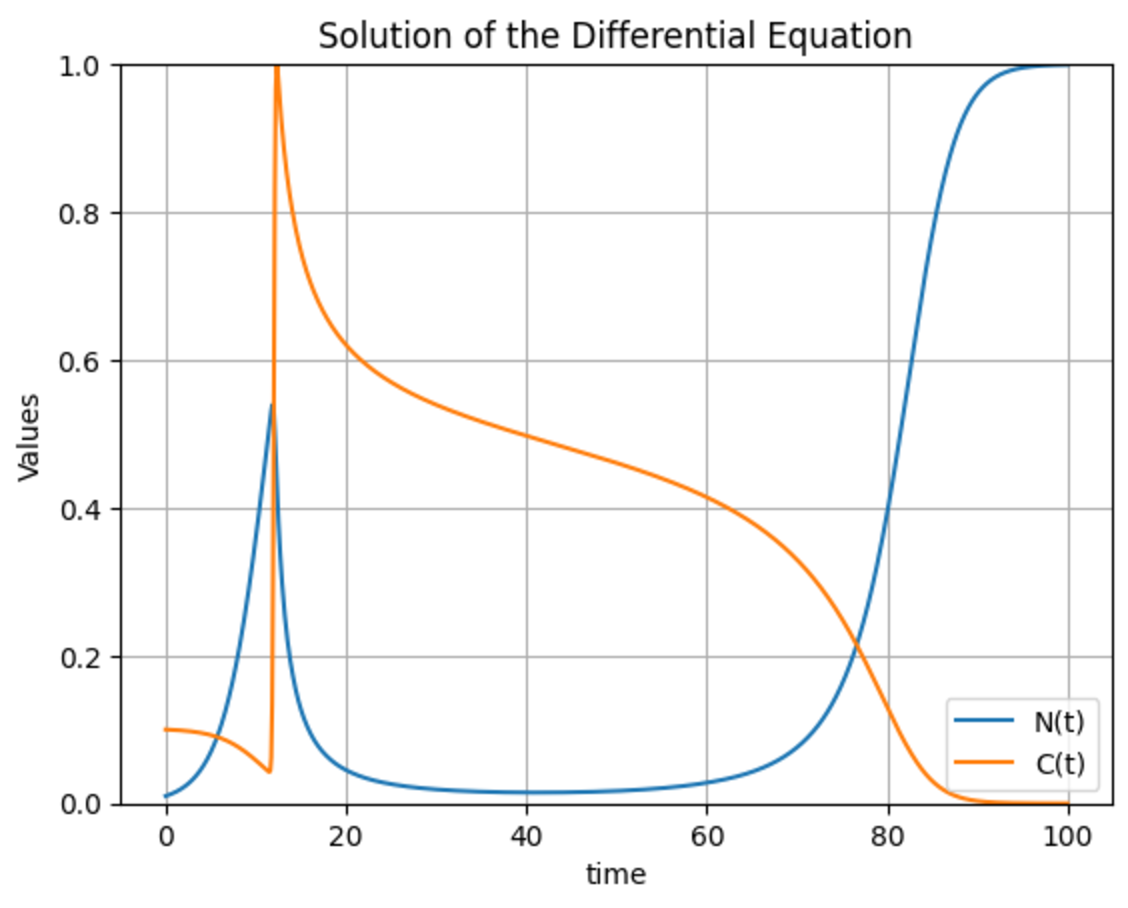}
\includegraphics[width=0.49\linewidth]{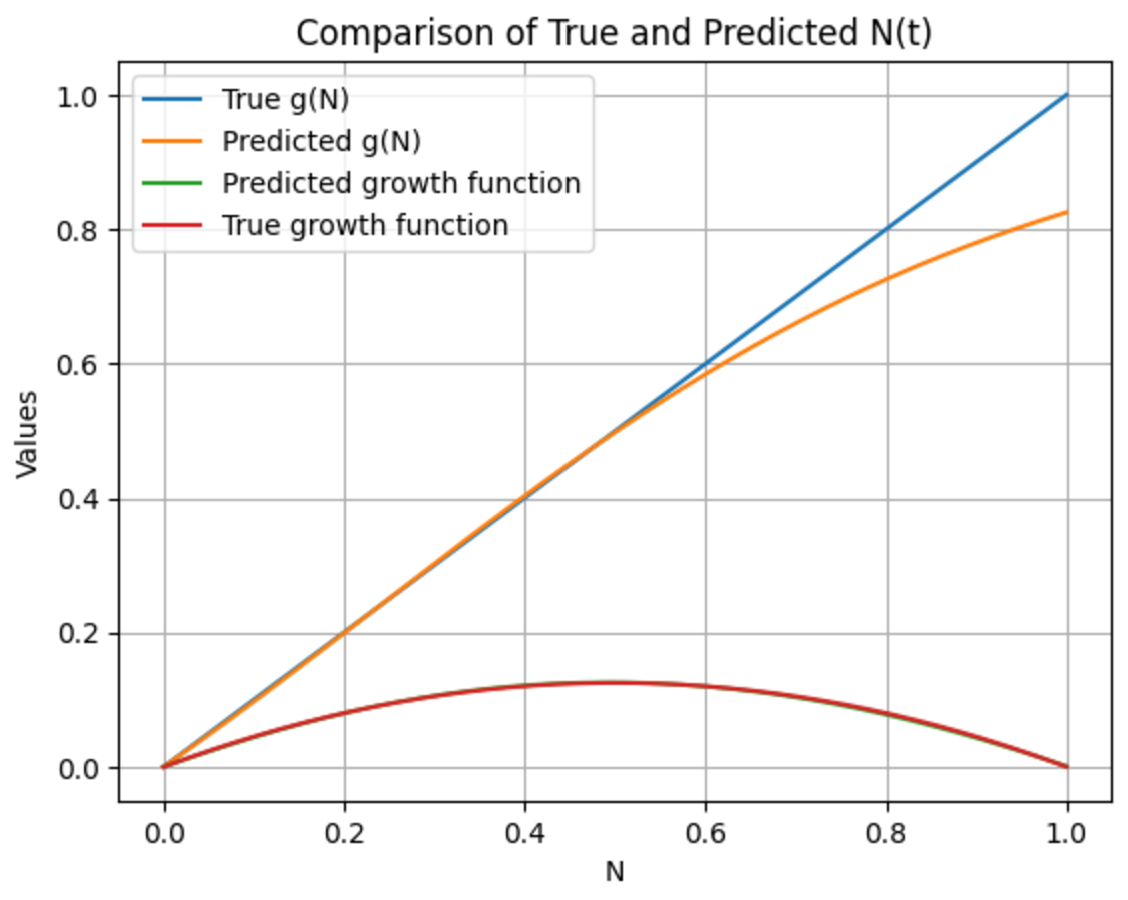}
\caption{\label{fig:modified problem solution} Results for the modified chemotherapy intervention problem. Left: Solution to the differential equations. Right: Comparison of the true and predicted growth term \(\Psi(N)\) and drug action \(u(N)\).}
\end{figure}

\subsection{Lotka-Volterra Predator-Prey System}

The Lotka-Volterra system provided an additional test case for the uniqueness conditions in a more complex setting. The solutions of the system with the given initial conditions are shown in Figure \ref{fig:Lotka-Volterra solution}.

The final estimated values from the network were \(\alpha = 0.99979\) and \(\gamma = 0.99970\). Since the \(\beta x\) and \(\delta y\) terms were approximated using neural networks, explicit values for \(\beta\) and \(\delta\) were not directly obtained. The comparisons between the true and predicted function values are shown in Figure \ref{fig:Lotka-Volterra solution} (right). Similar to the previous trials, the approximations begin to diverge from the true values as they approach the upper limit of the training data. The final loss achieved was \(4.341673 \times 10^{-5}\).

The oscillatory nature of the Lotka-Volterra system, combined with sufficient observational data, ensured that the uniqueness conditions in Theorem \ref{theorem1} were met, allowing the network to converge to the correct solution.

\begin{figure}[h]
\centering
\includegraphics[width=0.49\linewidth]{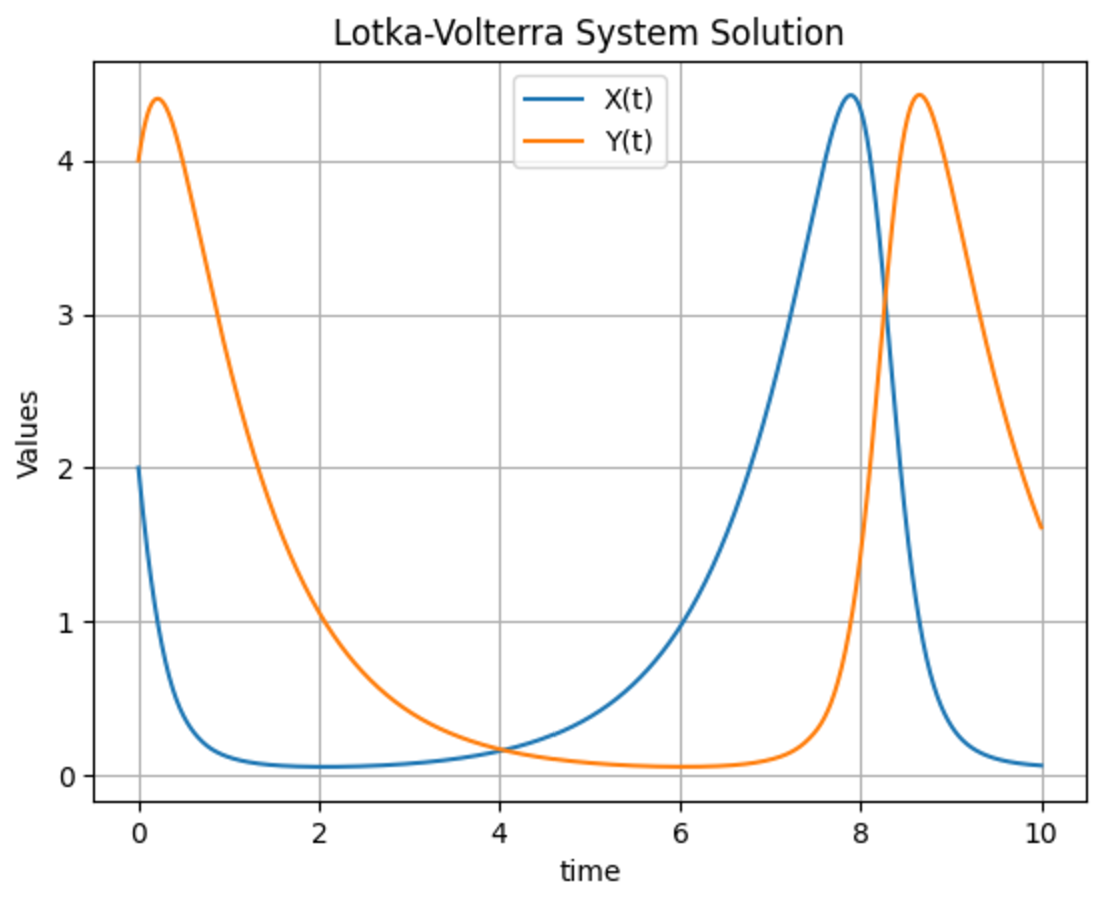}
\includegraphics[width=0.49\linewidth]{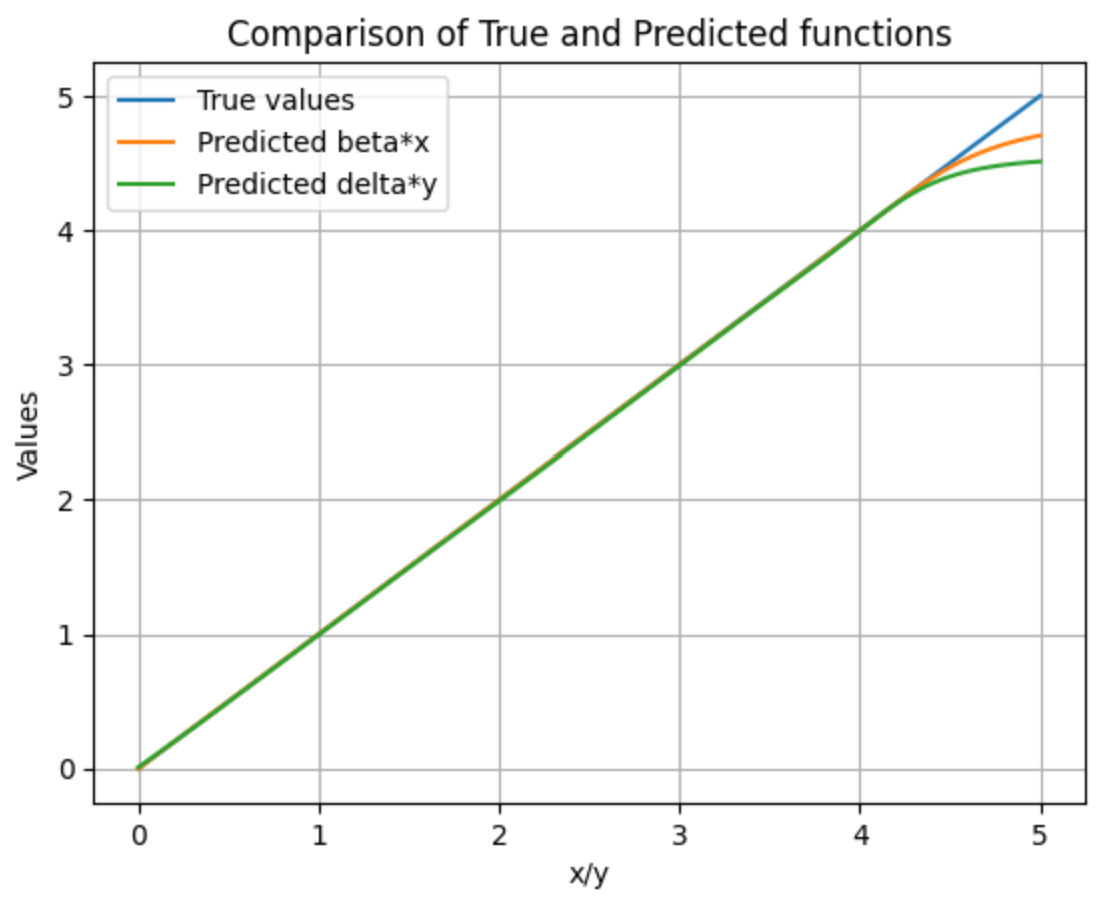}
\caption{\label{fig:Lotka-Volterra solution} Solution to the Lotka-Volterra system with the given initial conditions (left). Comparison between the true and predicted values of $\beta x$ and $\delta y$ (right).}
\end{figure}

The results demonstrate the effectiveness of the proposed methods in learning both unknown constants and functional forms. In the chemotherapy intervention problem, the network accurately recovered the true values of \(\beta\) and \(u(N)\), although the performance was sensitive to the availability of training data. The modified chemotherapy intervention problem confirmed that the network could handle more complex scenarios involving unknown growth terms, while maintaining accuracy where the uniqueness conditions were satisfied. Finally, the Lotka-Volterra system highlighted the applicability of the method to more complex ecological models, with accurate recovery of oscillatory dynamics and unknown parameters.

\subsection{Lotka-Volterra Predator-Prey System with Proportional Noise}

With substantial support for the uniqueness conditions in place, we now examine practical implementations of these results with UPINNs. The results obtained from the trial described in Section \ref{proportionalNoise} are shown in Tables \ref{tab:Proportional noise recovered constants comparison} and \ref{tab:Proportional noise losses}, which show the predicted values of the constants and final loss values from each trial respectively. The evaluation metrics listed in Table \ref{tab:Proportional noise losses} compare the final output of the trained model with the original uncorrupted data.

\begin{table}[h]
\centering
\begin{tabular}{|c|c|c|c|c|}
\hline
\%Noise & Pred. \(\alpha\) & Pred. \(\gamma\) & \(\alpha\) \%Error & \(\gamma\) \%Error \\
\hline
\hline
0\%   & 0.9978  & 0.9980  & 0.223\%   & 0.202\%   \\
\hline
5\%   & 0.9861   & 0.9974  & 1.385\%   & 0.257\%   \\
\hline
7.5\%   & 0.9811   & 0.9969  & 1.891\%   & 0.312\%   \\
\hline
10\%   & 0.9769   & 0.9964  & 2.314\%   & 0.364\%   \\
\hline
12.5\%   & 0.9732   & 0.9961  & 2.683\%   & 0.393\%   \\
\hline
15\%   & 0.9702   & 0.9954  & 2.983\%   & 0.458\%   \\
\hline
30\%   & 0.9586   & 0.9880  & 4.136\%   & 1.205\%   \\
\hline
\end{tabular}
\caption{The left most column displays the amount of injected noise for each trail and the rest show the corresponding final predictions of the model along with their percent errors.}
\label{tab:Proportional noise recovered constants comparison}
\end{table}

\begin{table}[h]
\centering
\begin{tabular}{|c|c|c|c|c|}
\hline
\% Noise & Data Loss & ODE Loss & \(R^2\) & MAPE  \\
\hline
\hline
0\%   & 1.075e-5   & 2.389e-4 & 0.999994 & 0.7192\%  \\
\hline
5\%   & 1.699e-4   & 4.794e-4 & 0.999895 & 2.2102\% \\
\hline
7.5\%   & 2.418e-4   & 6.802e-4 & 0.999851 & 2.0407\% \\
\hline
10\%   & 4.047e-4   & 8.936e-4  & 0.999750 & 1.9870\% \\
\hline
12.5\%   & 5.818e-4   & 1.232e-3  & 0.999641 & 2.1338\% \\
\hline
15\%   & 7.807e-4   & 1.887e-3 & 0.999519 & 2.4241\% \\
\hline
30\%   & 3.158e-3   & 7.212e-3 & 0.998010 & 5.9208\% \\
\hline
\end{tabular}
\caption{The left most column displays the amount of injected noise for each trail and the rest show the corresponding final losses, coefficients of determination (\(R^2\)), and Mean Absolute Percent Errors (MAPE) achieved by the model.}
\label{tab:Proportional noise losses}
\end{table}

These results demonstrate the robustness of UPINNs utilizing this functional form. In all trails, the percent error for the predicted parameters did not exceed 4.5\%. The low final data losses achieved by the models indicate that the physics-informed nature of the networks assisted in the prevention of overfitting to the random noise.

\subsection{Chemotherapy Intervention Problem with Varying Dataset Length}

We now examine the effects of varying dataset lengths on UPINNs utilizing this functional form. The results of the trial described in Section \ref{Varying dataset Lengths} are summarized in Tables \ref{tab:varying dataset length growth constant} and \ref{tab:varying dataset length evaluation metrics}.

\begin{table}[h]
\centering
\begin{tabular}{|c||c|c|}
\hline
Length&Pred. \(\beta\)&\(\beta\) \%Error\\
\hline
\hline
1024   & 1.9995   & 0.02256\%  \\
\hline
512   & 1.9995   & 0.02579\% \\
\hline
256  & 1.9994   & 0.03193\% \\
\hline
128   & 1.9996   & 0.02002\%  \\
\hline
64   & 1.9993   & 0.03340\% \\
\hline
32   & 1.9988   & 0.05940\%  \\
\hline
16   & 1.9934   & 0.33171\% \\
\hline
8   & 1.9729   & 1.35683\%  \\
\hline
4   & 0.8156   & 59.21943\% \\
\hline
\end{tabular}
\caption{The left most column displays the length of the training set for each trail and the rest show the corresponding predictions of \(\beta\) along with their percent error.}
\label{tab:varying dataset length growth constant}
\end{table}

\begin{table}[h]
\centering
\begin{tabular}{|c|c|c|c|c|}
\hline
Length&Data Loss&ODE Loss& \(R^2\) & MAPE\\
\hline
\hline
1024   & 3.302e-7   & 1.942e-5  & 0.999992   & 0.3934\%  \\
\hline
512   & 3.289e-7   & 1.895e-5 & 0.999992  & 0.3918\% \\
\hline
256  & 3.276e-7   & 1.823e-5 & 0.999992   & 0.3899\% \\
\hline
128   & 2.818e-7   & 1.670e-5 & 0.999993   & 0.3828\%  \\
\hline
64   & 2.957e-7   & 1.435e-5  & 0.999993   & 0.3750\% \\
\hline
32   & 2.513e-7   & 1.189e-5 & 0.999994   & 0.3595\%  \\
\hline
16   & 3.116e-6   & 6.719e-5 & 0.999925   & 0.3878\% \\
\hline
8   & 5.990e-5  & 2.996e-4 & 0.998567   & 2.2159\%  \\
\hline
4   & 9.752e-3   & 2.621e-3 & 0.766668   & 60.159\% \\
\hline
\end{tabular}
\caption{The left most column displays the length of the training set for each trail and the rest show the corresponding losses, coefficients of determination (\(R^2\)), and Mean Absolute Percent Errors (MAPE). }
\label{tab:varying dataset length evaluation metrics}
\end{table}

Throughout the trials with data set lengths from 1024 to 64, results remained relatively consistent with errors that hover around 0.025\% for the predicted \(\beta\). For data set lengths of 32 to 8, we observe a rapid increase in these errors, although the final result remains close to the true value. For the data set of length 4 we reach failure with a final error of 59.2\%. This behavior is mirrored in the losses, coefficients of determination, and mean absolute percent errors for these trails.

\section{Discussion} \label{Discussion}

This work demonstrated the practical utility of the proposed uniqueness conditions through five test cases: the chemotherapy intervention problem, its modified version, the Lotka-Volterra predator-prey system, the Lotka-Volterra system with proportional noise, and the chemotherapy intervention problem with varying dataset lengths. Each test case highlighted different aspects of the theoretical results and their applicability to inverse problems involving differential equations.

In the first three test cases, where the derivatives of the systems of differential equations were assumed to be known, the conclusions of the theorems discussed in Section \ref{Theory} were tested under their exact assumptions.

In the chemotherapy intervention problem, two scenarios were tested, in which the true drug actions were \(u(N) = N\) and \(u(N) = N^2\), with a true growth constant of \(\beta = 1\). In both cases, the model successfully recovered the true values of \(\beta\) and \(u(N)\) whenever the uniqueness conditions specified in Theorem \ref{theorem3} were met. The results confirmed that the approach effectively handles separable forms of the unknown functional terms, provided that the data adequately spans the functional and parametric space.

The modified chemotherapy intervention problem introduced a more complex scenario by treating the growth term \(\Psi(N)\) as entirely unknown. The model accurately learned both \(\Psi(N) = N(1-N)\) and \(u(N) = N\), with high fidelity in regions where the uniqueness conditions of Theorem \ref{theorem4} were satisfied. This experiment showcased the robustness of the proposed conditions in situations where multiple unknown terms interact within the governing equations.

The Lotka-Volterra predator-prey system also served as a test case involving oscillatory dynamics and nonlinear interactions. The model successfully recovered the hidden terms and parameters, including \(\alpha\) and \(\gamma\), and provided accurate estimated functions for \(\beta x\) and \(\delta y\). The oscillatory nature of the system ensured sufficient observational diversity, allowing the model to converge to a unique solution. These results demonstrated that the method can effectively handle coupled dynamics in ecological and similar systems.

One key observation across the first three test cases was the dependence of model performance on the availability of training data. In the chemotherapy intervention problem, the accuracy of the model decreased for values of \(N < 0.3\) and \(N < 0.6\) in the first and second trials, respectively, since these regions were not represented in the data. In the modified problem, a divergence was observed around \(N \approx 0.5\) where the uniqueness conditions were no longer fully satisfied. Similarly, in the Lotka-Volterra system, divergence is observed around \(N \approx 4.5\) where there is no longer information contained in the training data about the functions. These findings emphasize the importance of ensuring adequate data coverage to fully exploit the theoretical guarantees provided by the uniqueness conditions.

In the remaining test cases, the practical applicability of the theorems was tested. Instead of assuming that the derivatives of the systems of differential equations were known, UPINNs were fit to the trails, providing a surrogate approximation for the derivative through auto-differentiation.

In the Lotka-Volterra system with proportional noise, the performance of UPINNs utilizing Theorem \ref{theorem3} were tested with different amounts of injected noise. As the magnitude of noise was increased, we observed a consistent degradation in accuracy and loss. This shows that although performance suffers from noisy measurements, the method is still able to extract relevant information to a large extent. This observed robustness to noise aligns with previous findings for UPINNs \cite{podina2023universal}, indicating that the altered problem setting does not significantly impact performance.

In the chemotherapy intervention problem with varying dataset lengths, we tested UPINNs using Theorem~\ref{theorem3} across datasets of different sizes. We observed that model performance remained relatively stable as the number of training points decreased, until a threshold was reached beyond which the loss began to increase. These findings are qualitatively consistent with the results in \cite{podina2023universal}, where the authors show that UPINNs perform well on sparse datasets, provided that critical dynamics are adequately sampled. This suggests that our reformulated problem retains the robustness characteristics of UPINNs, even when data sparsity increases.

This work is positioned within the PINN framework established in \cite{raissi2019physics}. We demonstrated sufficient conditions such that the methods of PINNs and UPINNs \cite{raissi2019physics} can be applied simultaneously to solve inverse problems that would be intractable for each model independently. In the following, we discuss solutions to these problems from alternative frameworks.

Sparse Identification of Nonlinear Dynamical Systems (SINDy) \cite{brunton2016sindy} is a widely used framework for discovering governing equations directly from time-series data. It constructs a library \( \Theta(x) \) of candidate functions, such as polynomials, trigonometric terms, or other basis elements, and identifies a sparse vector of coefficients \( \xi \) such that the system dynamics can be expressed as \( \dot{x} = \Theta(x) \xi \). The success of SINDy relies on two key conditions: (i) that the true governing equation lies within the span of the candidate library, and (ii) that the active columns of \( \Theta(x) \) are sufficiently independent to allow for accurate sparse recovery.


The assumptions required by SINDy are different from those we establish in this work. This allows the method to successfully recover the symbolic expression for the differential equation, including the unknown component, even when our uniqueness conditions are not met. We believe that the reason SINDy is able to recover the system even when solutions to the inverse problem are not unique is due to condition (ii): sparse representability within the active columns of \( \Theta(x) \). While an alternative solution may lie within the span of the candidate library, if the condition is met, its representation will be of leaser sparsity and thus not be chosen. For the chemotherapy intervention problem, we conducted a test to confirm these observations. In our experiment, SINDy was able to converge to the correct answer when the uniqueness conditions we established were not met. Although, as shown in Section \ref{non-uniquness discussion}, a counterexample can be constructed where a different system has the same solution, these scenarios are not found in most physical systems \cite{brunton2016sindy}.

The SINDy and PINN/UPINN methods can be seen to be complementary to each other. Both approaches require different sets of assumptions that do not perfectly overlap and are not always met.

The Simformer \cite{gloeckler_all--one_2024} is a probabilistic diffusion model that can be used for both forward problems given a set of parameters and inverse problems, finding a probability distribution for the parameters conditioned on a set of data. FUSE \cite{lingsch_fuse_2024} is a framework built upon the Neural Operator architecture \cite{kovachki_neural_2024} that allows forward and inverse problems to be solved simultaneously. The inverse problem component of FUSE returns a probability distribution over the parameters. Both of these methods find probability distributions over numerical parameters. They do not explicitly find the governing equations followed by the data, instead opting for direct mappings between parameter values and solutions, hence they cannot explicitly find unknown functional terms of governing equations. 

It is also important to note that these models are purely data driven and thus require orders of magnitude more training samples and model parameters than PINNs and SINDy. Where PINNs and SINDy can solve the inverse problem after being trained on as little as one trajectory through the state space, Simformers and FUSE must be trained on hundreds. After Simformers and FUSE have concluded their training phases, they are able to make inferences on trajectories and corresponding parameters for many samples not seen during training, where PINNs and SINDy would be required to retrain on the new data. 

The differences in these models reflect the variety of applications they aim to address. When choosing a model for an inverse problem, all of these models may be considered and chosen based on the constraints and requirements of the setting.





\section{Conclusions and Future Directions}

This study explored the application of UPINNs for solving inverse problems in systems of differential equations. By deriving and applying sufficient conditions for uniqueness, this work demonstrated the feasibility of simultaneously learning unknown constants and functional terms. The experimental validation included the chemotherapy intervention toy problem, its modified version, the Lotka-Volterra predator-prey system, the Lotka-Volterra system with proportional noise, and the chemotherapy intervention problem with varying data set lengths, each of which provided insights into the practical applications of the proposed methods.

The results showed that the uniqueness conditions derived in Theorems \ref{theorem1} to \ref{theorem4} offer a strong theoretical foundation for solving inverse problems involving complex systems. The experiments illustrated the ability of the method to accurately recover unknown coefficients and functional forms in diverse scenarios, provided that the observational data sufficiently spanned the relevant domains. These findings underscore the application of UPINNs to address limitations of traditional data-driven methods, particularly in cases where the governing equations include poorly understood terms or are influenced by sparse data.

Several areas remain open for further exploration. Expanding the theoretical framework to encompass more general functional forms would enhance the applicability of the method to a broader class of problems. In addition, developing techniques to infer these functional forms directly from data, particularly when the governing equations are partially or entirely unknown, would represent a valuable extension.
 Broadening the scope of applications to include fields such as climate modeling, fluid dynamics, and biomedical systems would also provide opportunities to demonstrate the versatility and impact of these methods. This work establishes a foundation for combining physics-informed and universal neural networks to address inverse problems in scientific and engineering domains. Future research in this area could significantly enhance the capability of machine learning models to uncover hidden dynamics, inform decision-making, and advance the understanding of complex systems.

\section*{Impact Statement}
This paper presents work whose goal is to advance the field of Machine Learning. There are many potential societal consequences of our work, none which we feel must be specifically highlighted here.

\section*{Acknowledgments}  
SM acknowledges the support of the Natural Sciences and Engineering Research Council of Canada (NSERC) through the Undergraduate Student Research Awards (USRA) program. MK acknowledges the support of the Natural Sciences and Engineering Research Council of Canada (NSERC) through the Discovery Grant program. 

\bibliographystyle{plainnat} 

\bibliography{bibliography} 
\nocite{*}

\appendix
\section{Proofs of Theorems \ref{theorem3} and \ref{theorem4}}
\addcontentsline{toc}{section}{Appendix}

\begin{remark}
    To make the proof easier to read and to accommodate the margins, we have not included the term \(d(x)\) in the proofs below. Including this term would slightly change the expressions for the unknown terms, but would have no effect on the magnitude of the correction term.
\end{remark}

\begin{theorem*}[Theorem \ref{theorem3} restated]
With the data sets:
\[
\dot{\chi} = [\dot{x}_1, \dots, \dot{x}_m], \quad \chi = [x_1, \dots, x_m], 
\]
For the system described in Theorem \ref{theorem1}, suppose that \( u \) is Lipschitz continuous with Lipschitz constant \( L \).  
Furthermore, suppose that we have \( i,j \in \{1, \dots, m\} \) and \( D \geq 0 \) such that \( \|y_i - y_j\| \leq D \), \( C(x_i) \neq C(x_j) \), and \( g(y_i) \neq 0 \), then for \( i \in \{1, \dots, m\} \), \(\beta\) and \(u(y_i)\) are guaranteed to be within intervals with radii continuously dependent on \(D\).
\end{theorem*}
\begin{proof}
W.L.O.G. assume \( C(x_i) > C(x_j) \).
From the governing equation, we have that:
\[
(\dot{x}_i)_q = \beta g(y_i) + C(x_i)u(y_i)
\]
and
\[
(\dot{x}_j)_q = \beta g(y_j) + C(x_j)u(y_j)
\]
Subtracting these equations yields:
\[
(\dot{x}_i - \dot{x}_j)_q = \beta (g(y_i) - g(y_j)) + C(x_i)u(y_i) - C(x_j)u(y_j)
\]
Now, by hypothesis:
\[
\|y_i - y_j\| \leq D \leftrightarrow -D + y_i \leq y_j \leq D + y_i 
\]
\[
\Rightarrow y_j \in [-D + y_i,D + y_i]^k := I
\]
Note that the above inequalities apply element wise.
We will assume that \(C(x_j) > 0\). If it is not, then it flips the inequalities, which has no effect because we find an interval for \(\beta.\)

We now examine the largest possible value for the right-hand side of the equation.
\[
\Rightarrow (\dot{x}_i - \dot{x}_j)_q 
\]
\[
\leq \beta (g(y_i) - g(y_j)) + C(x_i)u(y_i) - \text{min}_{Y \in I}(C(x_j)u(Y))
\]
But since \(Y \in I\) and \(u\) is Lipschitz, we have that
\[
||u(y_i) - u(Y)|| \leq L||y_i - Y|| \leq LD 
\]
\[
\leftrightarrow -LD + u(y_i) \leq u(Y) \leq LD + u(y_i)
\]
Thus
\[
(\dot{x}_i - \dot{x}_j)_q 
\]
\[
\leq \beta (g(y_i) - g(y_j)) + C(x_i)u(y_i) - (C(x_j)(-LD + u(y_i))
\]
\[
=\beta (g(y_i) - g(y_j)) + u(y_i)(C(x_i) - C(x_j)) + C(x_j)LD
\]
Applying the same thing to get "\(\geq\)" yields:
\[
\beta (g(y_i) - g(y_j)) + u(y_i)(C(x_i) - C(x_j)) - C(x_j)LD 
\]
\[
\leq (\dot{x}_i - \dot{x}_j)_q 
\]
\[
\leq \beta (g(y_i) - g(y_j)) + u(y_i)(C(x_i) - C(x_j)) + C(x_j)LD
\]
Rearranging yields:
\[
\frac{(\dot{x}_i - \dot{x}_j)_q - \beta (g(y_i) - g(y_j)) - C(x_j)LD}{C(x_i) - C(x_j)} \leq u(y_i) 
\]
\[
\leq \frac{(\dot{x}_i - \dot{x}_j)_q - \beta (g(y_i) - g(y_j)) + C(x_j)LD}{C(x_i) - C(x_j)}
\]

Rearranging the original equations yields:
\[
\beta = \frac{(\dot{x}_i)_q - C(x_i)u(y_i)}{g(y_i)}
\]
Now if we plug in our found range for \(u(y_i)\), we get that:

\[
\beta =  \frac{(C(x_i)\dot{x}_j - C(x_j)\dot{x}_i)_q}{g(y_i)(C(x_i) - C(x_j)) - (C(x_i) - C(x_j))} + \text{ Correction}
\]

Where 
\[
\text{Correction } \in \bigg[\frac{-C(x_j)C(x_i)LD}{g(y_i)(C(x_i) - C(x_j)) - (C(x_i) - C(x_j))}, \dots
\]
\[
\dots \frac{C(x_j)C(x_i)LD}{g(y_i)(C(x_i) - C(x_j)) - (C(x_i) - C(x_j))}\bigg]
\]
Now we rearrange the original equation to get:
\[
u(y_p) = \frac{(\dot{x}_p)_q - \beta g(y_p)}{C(z_p)} 
\]
\[
= \frac{(\dot{x}_p)_q - \frac{(C(x_i)\dot{x}_j - C(x_j)\dot{x}_i)_q}{g(y_i)(C(x_i) - C(x_j)) - (C(x_i) - C(x_j))} g(y_p)}{C(z_p)} - \dots
\]
\[
\dots \frac{\text{Correction}(g(y_p))}{C(z_p)}
\]
for all \(p \in \{1,\dots,m\}\).\\
Thus, for \( i \in \{1, \dots, m\} \), \(\beta\) and \(u(y_i)\) are guaranteed to be within intervals with radii continuously dependent on \(D\).
\end{proof}

\begin{remark}
If we define
\[
\bar{\beta} := \beta - \text{Correction}
\]
we get that
\[
|\beta - \bar{\beta}| = |\text{Correction}|
\]
Now since we have that
\[
\text{Correction} \in \bigg[\frac{-C(x_j)C(x_i)LD}{g(y_i)(C(x_i) - C(x_j)) - (C(x_i) - C(x_j))}, \dots
\]
\[
\dots \frac{C(x_j)C(x_i)LD}{g(y_i)(C(x_i) - C(x_j)) - (C(x_i) - C(x_j))}\bigg]
\]
we know that
\[
|\text{Correction}| \leq \bigg|\frac{C(x_j)C(x_i)LD}{g(y_i)(C(x_i) - C(x_j)) - (C(x_i) - C(x_j))}\bigg|
\]
Thus
\[
|\beta - \bar{\beta}| \leq \bigg|\frac{C(x_j)C(x_i)LD}{g(y_i)(C(x_i) - C(x_j)) - (C(x_i) - C(x_j))}\bigg|
\]
performing a similar procedure of \(u\) yields:
\[
|u(y_p) - \bar{u}(y_p)| \leq \bigg|\frac{g(y_p)}{C(z_p)}\bigg||\beta - \bar{\beta}|
\]
for all \(p \in \{1,\dots,m\}\).
\end{remark}

\begin{theorem*}[Theorem \ref{theorem4} restated]
With the data sets:
\[
\dot{\chi} = [\dot{x}_1, \dots, \dot{x}_m], \quad \chi = [x_1, \dots, x_m], 
\]
For the system described in Theorem \ref{theorem2}, suppose that \(g(y)\) and \( u(y) \) are Lipschitz continuous with Lipschitz constants \( L_1\) and \(L_2\), respectively.  
Furthermore, suppose that we have \( i,j \in \{1, \dots, m\} \) and \( D \geq 0 \) such that \( \|y_i - y_j\| \leq D \) and \( C(x_i) \neq C(x_j) \), then \(g(y_i)\) and \(u(y_i)\) are guaranteed to be within intervals with radii continuously dependent on \(D\).
\end{theorem*}
\begin{proof}
plugging into the governing equations yields:
\[
(\dot{x}_i)_q = g(y_i) + C(x_i)u(y_i),
\]
and
\[
 (\dot{x}_j)_q = g(y_j) + C(x_j)u(y_j)
\]
subtracting the two equations from each other gives us:
\[
(\dot{x}_i - \dot{x}_j)_q = g(y_i) - g(y_j) + C(x_i)u(y_i) - C(x_j)u(y_j)
\]
Now, by hypothesis:
\[
\|y_i - y_j\| \leq D \leftrightarrow -D + y_i \leq y_j \leq D + y_i
\]
\[
\Rightarrow y_j \in [-D + y_i,D + y_i]^k := I
\]
Note that the above inequalities apply element wise.
We will assume that \(C(x_j) > 0\). If it is not, then it flips the inequalities, which has no effect because we find an interval for \(\beta.\)

We now examine the largest possible value for the right-hand side of the equation.
\[
(\dot{x}_i - \dot{x}_j)_q \leq g(y_i) - g(y_j) + C(x_i)u(y_i) - \text{min}_{Y \in I}(C(x_j)u(Y))
\]
But since \(Y \in I\) and \(u\) is Lipschitz, we have that
\[
||u(y_i) - u(Y)|| \leq L_2||y_i - Y|| \leq L_2D 
\]
\[
\leftrightarrow -L_2D + u(y_i) \leq u(Y) \leq L_2D + u(y_i)
\]
Thus
\[
(\dot{x}_i - \dot{x}_j)_q \leq g(y_i) - g(y_j) + C(x_i)u(y_i) - (C(x_j)(-L_2D + u(y_i))
\]
\[
=g(y_i) - g(y_j) + u(y_i)(C(x_i) - C(x_j)) + C(x_j)L_2D
\]
Now since \(g(y)\) has Lipschitz constant \(L_1\), we have: 
\[
g(y_i) - g(y_j) + u(y_i)(C(x_i) - C(x_j)) + C(x_j)L_2D 
\]
\[
\leq \|g(y_i) - g(y_j)\| + u(y_i)(C(x_i) - C(x_j)) + C(x_j)L_2D
\]
\[
\leq L_1\|y_i - y_j\| + u(y_i)(C(x_i) - C(x_j)) + C(x_j)L_2D
\]
\[
\leq L_1D + u(y_i)(C(x_i) - C(x_j)) + C(x_j)L_2D 
\]
\[
= u(y_i)(C(x_i) - C(x_j)) + D(L_1 + C(x_j)L_2)
\]
Applying the same thing to get "\(\geq\)" yields:
\[
u(y_i)(C(x_i) - C(x_j)) - D(L_1 + C(x_j)L_2)\leq (\dot{x}_i - \dot{x}_j)_q 
\]
\[
\leq u(y_i)(C(x_i) - C(x_j)) + D(L_1 + C(x_j)L_2) 
\]
Rearranging yields:
\[
\frac{(\dot{x}_i - \dot{x}_j)_q - D(L_1 + C(x_j)L_2)}{C(x_i) - C(x_j)} \leq u(y_i) 
\]
\[
\leq \frac{(\dot{x}_i - \dot{x}_j)_q + D(L_1 + C(x_j)L_2)}{C(x_i) - C(x_j)}
\]

Thus there exists a 
\[\text{Correction } \in \bigg[\frac{(- D(L_1 + C(x_j)L_2)}{C(x_i) - C(x_j)},\frac{D(L_1 + C(x_j)L_2)}{C(x_i) - C(x_j)}\bigg]
\]
such that
\[
u(y_i) = \frac{(\dot{x}_i - \dot{x}_j)_q }{C(x_i) - C(x_j)} + \text{Correction}
\]
Rearranging the original equations yields:
\[
g(y_i) = (\dot{x}_i)_q - C(x_i)u(y_i) 
\]
\[
= (\dot{x}_i)_q - C(x_i)(\frac{(\dot{x}_i - \dot{x}_j)_q }{C(x_i) - C(x_j)} + \text{Correction})
\]
\[
 = (\dot{x}_i)_q - \frac{C(x_i)(\dot{x}_i - \dot{x}_j)_q }{C(x_i) - C(x_j)} + C(x_i)(\text{Correction})
\]
Thus, \(g(y_i)\) and \(u(y_i)\) are guaranteed to be within intervals with radii continuously dependent on \(D\).
\end{proof}

\begin{remark}
If we define
\[
\bar{u}(y_p) := u(y_p) - \text{Correction}
\]
for all \(p \in \{1,\dots,m\}\),
we get that
\[
|u(y_p) - \bar{u}(y_p)| \leq |\text{Correction}|
\]
Now since we have that
\[
\text{Correction} \in \bigg[\frac{(- D(L_1 + C(x_j)L_2)}{C(x_i) - C(x_j)},\frac{D(L_1 + C(x_j)L_2)}{C(x_i) - C(x_j)}\bigg]
\]
we know that
\[
|\text{Correction}| \leq \bigg|\frac{D(L_1 + C(x_j)L_2)}{C(x_i) - C(x_j)}\bigg|
\]
Thus
\[
|u(y_p) - \bar{u}(y_p)| \leq \bigg|\frac{D(L_1 + C(x_j)L_2)}{C(x_i) - C(x_j)}\bigg|
\]
performing a similar procedure of \(\beta\) yields:
\[
|\beta - \bar{\beta}| \leq \bigg|\frac{C(x_i)D(L_1 + C(x_j)L_2)}{C(x_i) - C(x_j)}\bigg|
\]
\end{remark}

\end{document}